\documentclass[runningheads]{llncs}
\usepackage[T1]{fontenc}
\usepackage{graphicx}
\usepackage{color}
\usepackage{lipsum}            
\usepackage{float}             
\usepackage{wrapfig}           
\usepackage{xcolor}            
\usepackage{url}               
\usepackage{hyperref}          
\usepackage{booktabs}          
\usepackage{multirow}
\usepackage{amsmath, amssymb}
\usepackage{mathtools}         
\usepackage{nicefrac}          
\usepackage{microtype}         
\usepackage{type1cm}           
\usepackage{bm}                
\usepackage{dsfont}            
\usepackage{cases}             
\usepackage{array}             

\begin{document}
\title{Symmetry-quotient Flatness and Generalization}
%
%
\author{Taiki Miyagawa \orcidID{0000-0001-7651-6706}}
\authorrunning{T. Miyagawa}
%
\institute{NEC Corporation, Japan
\email{miyagawataik@nec.com}\\
\url{https://sites.google.com/view/taiki-miyagawa-kanaheinousagi}}
\maketitle              
\begin{abstract}
This paper develops a theorem-level pipeline in symmetry-quotient settings: quotient linear stability implies quotient flatness, quotient flatness implies input smoothness, and input smoothness yields generalization under local covering assumptions.
Flatness is often associated with generalization, and Stochastic Gradient Descent (SGD) is frequently viewed as implicitly biased toward flat solutions.
However, standard flatness measures are typically defined in the raw parameter space and are therefore not invariant under function-preserving symmetries such as positive rescaling.
We develop a symmetry-aware theory of quotient flatness, quotient linear stability, input smoothness, and generalization on quotient spaces of neural-network parameters.
For square loss and models equipped with function-preserving group actions, we define \textit{quotient flatness} as the trace of the Hessian of the empirical loss on the regular quotient manifold.
We show that quotient flatness controls input smoothness through a quotient-space analogue of the flatness-to-smoothness argument.
We also prove that one-step mean-square quotient linear stability of the linearized SGD dynamics implies an explicit quotient-flatness bound in terms of the batch size and learning rate, and extend this analysis to higher-order tensor moments.
Finally, under local covering and boundedness assumptions, we derive population generalization bounds in terms of quotient flatness and, consequently, in terms of quotient linear stability.

\end{abstract}

\section{Introduction}
Stochastic Gradient Descent (SGD) often finds solutions that generalize well in highly over-parameterized and nonconvex neural-network models. A widely held explanation is that SGD is biased toward flatter minima and that flatter minima generalize better \cite{dziugaite2017computing,hochreiter1997flat,jiang2020fantastic,keskar2017large,ma2021linear}. Despite the influence of this view, two conceptual issues remain unresolved.

The first issue is that the mechanism connecting flatness to generalization is not always mathematically explicit. A particularly clear theoretical route was established by \cite{ma2021linear}. Using the multiplicative structure of the first layer, they showed that parameter-space flatness, and more generally linear stability properties of SGD, imply control of input-side Sobolev seminorms. Under suitable covering assumptions on the data distribution, these smoothness estimates yield generalization bounds. Their analysis therefore provides an explicit pipeline from optimization-side stability and geometry to function-side regularity and generalization. However, this pipeline has not been formulated in a symmetry-quotient geometry.

The second issue is that flatness measured in the raw parameter space is generally not invariant under function-preserving symmetries. Neural networks commonly admit such symmetries, including positive rescaling of hidden units and permutations within a hidden layer. As emphasized by \cite{pittorino2022deep}, these symmetries can make parameter-space distance and flatness measures imprecise or misleading. Consequently, even when flatness appears to correlate with generalization, it remains unclear whether the relevant quantity is raw parameter-space geometry or a symmetry-reduced one. A qualitative visualization of this distinction is shown in Fig.~\ref{fig:raw-vs-quotient-landscape}. These observations lead to the following question: \textit{Can one formulate a mathematically explicit flatness-to-generalization theory in which flatness is defined only after quotienting out function-preserving symmetries?}

In this paper, we answer this question in the affirmative. Our central idea is to replace raw parameter-space flatness with \textit{quotient flatness}, defined on a quotient manifold obtained by removing function-preserving symmetries. We show that this quantity admits a natural Hessian-trace characterization and supports a symmetry-aware analogue of the flatness-to-smoothness-to-generalization pipeline \cite{ma2021linear}. More precisely, we consider models with a smooth function-preserving group action and study the empirical loss on the corresponding regular quotient manifold. For interpolation solutions under the square loss, we prove that quotient flatness equals the average squared norm of the quotient gradients evaluated on the training data. This identity is the basic geometric bridge for the subsequent analysis.

We next prove that quotient flatness controls \textit{input smoothness}, not quotient stability. This step relies on two structural ingredients: a multiplicative first-layer representation of the form
$
f(x,W)=\tilde f(W_1x,W_2),
$
and a local gauge-fixing section that allows derivatives with respect to first-layer parameters to be compared with quotient derivatives. This yields a symmetry-aware counterpart of the flatness-to-smoothness argument of \cite{ma2021linear}.

We then connect quotient flatness to the local behavior of SGD. We define quotient linearized SGD on the tangent space of the quotient manifold at an interpolation solution and prove that one-step mean-square quotient linear stability implies an explicit upper bound on quotient flatness. Thus, the stability result has the direction
$
\text{quotient linear stability}
\Rightarrow
\text{quotient flatness},
$
rather than the reverse. We further derive tensor recursions for higher-order quotient moments, which yield higher-order input-smoothness estimates.

Finally, under local covering assumptions on the data distribution together with local boundedness conditions around the training data, we obtain generalization bounds from quotient flatness and, by combining the preceding stability-to-flatness result, from quotient linear stability. Taken together, these results establish the theorem-level pipeline
$
\text{quotient linear stability}
\Rightarrow
\text{quotient flatness}
\Rightarrow
\text{input smoothness}
\Rightarrow
\text{generalization}.
$
Our contribution is therefore not merely to redefine flatness in a symmetry-invariant way, but to show that the optimization-to-generalization mechanism can be reconstructed on symmetry-reduced geometry. In this sense, the paper provides a mathematically explicit symmetry-aware version of the flatness narrative.

\begin{figure}[tb]
    \begin{center}
    \includegraphics[width=\textwidth]{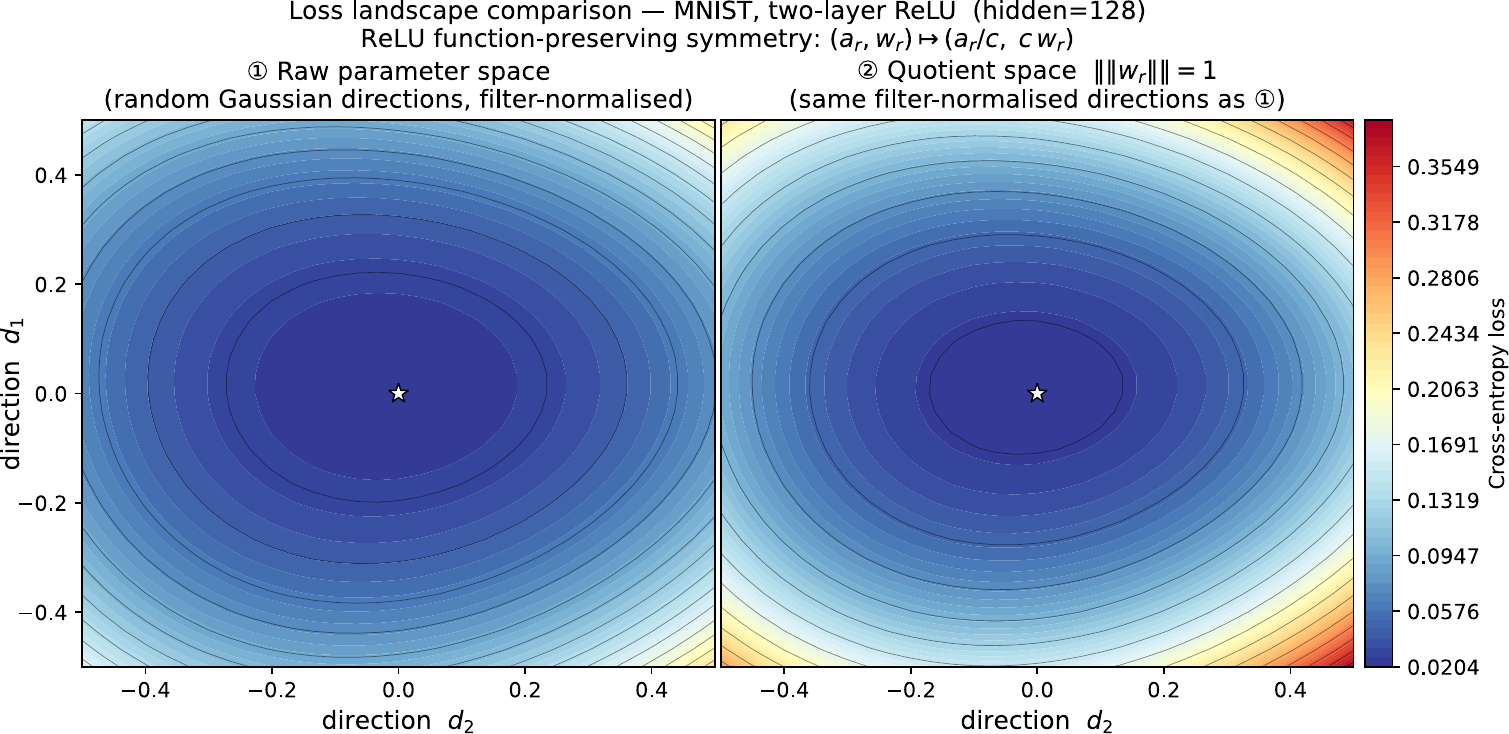}
    \end{center}
    \caption{
        \textbf{Raw versus symmetry-quotient loss landscapes for two-layer ReLU network trained on MNIST.}
        The model has hidden width $128$ and admits the function-preserving positive-rescaling symmetry
        $(a_r,w_r)\mapsto (a_r/c,cw_r)$.
        The left panel shows a two-dimensional cross-entropy loss slice in the raw parameter space along random filter-normalized directions, whereas the right panel shows the corresponding quotient-space slice after removing the scale degree of freedom.
        The star denotes the trained solution.
        The figure illustrates how quotienting removes representation-dependent scaling effects from local loss-geometry visualization.
    } \label{fig:raw-vs-quotient-landscape}
\end{figure}

\section{Related Work} \label{sec:related-work}
\subsection{Flatness, Generalization, and Implicit Regularization}
The link between flat minima and generalization has been studied since early work on minimum description length and PAC-Bayes perspectives \cite{hochreiter1997flat,dziugaite2017computing}. Empirical studies further show that SGD often selects solutions that generalize better when they are flatter under suitable local measures \cite{keskar2017large,jiang2020fantastic}, motivating flatness-aware optimization, margin-based views, and implicit-regularization analyses \cite{foret2021sharpnessaware,baldassi2021unveiling,wu2023implicit}. However, naively defined parameter-space flatness is representation-dependent: rescaling can change sharpness without changing the implemented function \cite{tsuzuku2020normalized}.

A key theoretical route was developed by \cite{ma2021linear}. They showed that, through the multiplicative structure of the first layer, parameter-space flatness and higher-order linear stability of SGD control input-side Sobolev seminorms, which in turn yield generalization bounds under distributional covering assumptions. Our work follows this flatness-to-smoothness-to-generalization pipeline, but reformulates it on a symmetry-reduced quotient loss geometry.

\subsection{Symmetry, Quotient Geometry, and Function-space Viewpoints}
Neural networks admit function-preserving symmetries, including positive rescaling in ReLU-type networks and hidden-unit permutations \cite{neyshabur2015path,meng2018gsgd,grigsby2023hidden}. These symmetries make raw parameter-space distance, connectivity, and flatness difficult to interpret. Fig.~\ref{fig:raw-vs-quotient-landscape} illustrates this issue: the raw slice contains redundant scale directions, whereas the quotient slice removes them.

This problem was emphasized by \cite{dinh2017sharp}, who showed that sharpness can be changed by reparameterization. \cite{pittorino2022deep} gave an empirical symmetry-reduced treatment by normalizing hidden units, aligning permutation-equivalent networks, and studying flat regions and geodesic paths after symmetry removal. Their results suggest that meaningful geometric structure emerges only after quotienting out representational redundancy. Building on this viewpoint, we define quotient flatness intrinsically on a quotient manifold and prove theorem-level links from quotient flatness to quotient stability, input smoothness, and generalization.

\section{Preliminaries} \label{sec:setup}
Let $\Theta$ be a smooth parameter manifold, let $\mathcal X \subset \mathbb R^d$ be the input space, and let $f_\theta : \mathcal X \to \mathbb R$ be a scalar-output model. Let $G$ be a Lie group or a finite group acting smoothly on $\Theta$. We assume that the action is function-preserving, namely
$
f_{g\cdot \theta}(x)=f_\theta(x)
$
for every $\theta \in \Theta$, $g\in G$, and $x\in\mathcal X$.
Let $\Theta_{\mathrm{reg}} \subset \Theta$ be a regular stratum on which the action is free and proper, and define the quotient manifold
$
\mathcal M := \Theta_{\mathrm{reg}}/G.
$
We denote by $\pi:\Theta_{\mathrm{reg}} \to \mathcal M$ the quotient map. Since the action is function-preserving, the model descends to a quotient model
$
\bar f_{[\theta]}(x):=f_\theta(x),
[\theta]=\pi(\theta)\in\mathcal M.
$
Given training data $\{(x_i,y_i)\}_{i=1}^n$, define the empirical square loss on the quotient by
$
\bar L([\theta]) := \frac{1}{2n}\sum_{i=1}^n (\bar f_{[\theta]}(x_i)-y_i)^2.
$
A point $[\theta_\star]\in\mathcal M$ is called an \textit{interpolation solution} if
$
\bar f_{[\theta_\star]}(x_i)=y_i
$
for all $i=1,\dots,n$.

\paragraph{Riemannian structure on the quotient.}
Equip $\mathcal M$ with a fixed Riemannian metric $g^Q$. 
For every smooth function $u:\mathcal M \to \mathbb R$, let $du$ denote its differential, let $\nabla^Q u := \operatorname{grad}_{\mathcal M}u$ denote its Riemannian gradient characterized by
$
g^Q_{[\theta]}(\nabla^Q u([\theta]),v)=du_{[\theta]}(v)
$
for all $[\theta]\in\mathcal M$ and $v\in T_{[\theta]}\mathcal M$, and let $\mathrm{Hess}_{\mathcal M}u$ denote its Riemannian Hessian, viewed as the $(0,2)$-tensor
$
\mathrm{Hess}_{\mathcal M}u(v,w)
:=
g^Q\bigl((\nabla^{\mathrm{LC}})_v \nabla^Q u,w\bigr)
$,
where $\nabla^{\mathrm{LC}}$ is the Levi--Civita connection of $g^Q$. Throughout, all traces and norms on $\mathcal M$ are taken with respect to $g^Q$.
We can now define the \textit{quotient flatness}:

\begin{definition}[Quotient flatness]
Let $[\theta_\star]\in\mathcal M$ be an interpolation solution. The quotient flatness is defined by
$
\mathrm{Flat}^Q([\theta_\star])
:=
\mathrm{Tr}_{g^Q}\bigl(\mathrm{Hess}_{\mathcal M}\bar L([\theta_\star])\bigr).
$
\end{definition}

For each training point $x_i$, define
$
\bar f_i([\theta]) := \bar f_{[\theta]}(x_i)
$
and
$
a_i^Q := \nabla^Q \bar f_i([\theta_\star]) \in T_{[\theta_\star]}\mathcal M.
$
Then, the following identity holds for the quotient flatness:

\begin{proposition}[Quotient flatness identity] \label{prop:general-quotient-flatness-identity}
    Let $[\theta_\star]\in\mathcal M$ be an interpolation solution. Then,
    $
    \mathrm{Flat}^Q([\theta_\star])
    =
    \frac{1}{n}\sum_{i=1}^n \|a_i^Q\|_{g^Q}^2.
    $
    The proof is provided in Appendix~\ref{app:proof-prop-general-quotient-flatness-identity}.
\end{proposition}

Prop.~\ref{prop:general-quotient-flatness-identity} shows that, at interpolation, quotient flatness is exactly the average squared $g^Q$-norm of the quotient gradients of the samplewise outputs. 
Thus, it is the symmetry-aware counterpart of the standard Hessian-trace notion of flatness.

For the square loss, the interpolation assumption removes the residual-dependent second-derivative term from the Hessian. 
As a result, the quotient Hessian trace becomes exactly the empirical average of the squared quotient-gradient norms. 
Without interpolation, the identity would contain an additional curvature term and would no longer reduce to a purely first-order quantity.





\section{Quotient Flatness Implies Input Smoothness}
Around an interpolation point $[\theta_\star]$, assume there is a smooth local gauge-fixing section
$
s:U\subset\mathcal M \to \Theta_{\mathrm{reg}},
\qquad
s([\theta])=(W_1,\vartheta),
$
such that the model admits the first-layer factorization
$
f_{s([\theta])}(x)=\tilde f(W_1x,\vartheta),
$
where $\vartheta$ collects all gauge-fixed parameters except $W_1$.
The section selects a regular representative, making $W_1$ well defined on the quotient, while the factorization converts $W_1$-derivatives into input derivatives.
Assume also that the quotient gradient controls the first-layer gradient: for some $C_{\mathrm{sec}}>0$,
\begin{align}
    \|\nabla_{W_1} f_{s([\theta_\star])}(x_i)\|_F
    \le
    C_{\mathrm{sec}}
    \|\nabla^Q \bar f_{[\theta_\star]}(x_i)\|_{g^Q},
    \qquad
    i=1,\dots,n.
\end{align}
Let $\Lambda_\star := \|W_{1,\star}\|_{\mathrm{op}}$, where $W_{1,\star}$ is the first-layer block of $s([\theta_\star])$, and define
$
\rho := \min_{1\le i\le n}\|x_i\|_2.
$
Then quotient flatness implies input smoothness:

\begin{theorem}[Quotient flatness implies input smoothness] \label{thm:general-quotient-flatness-implies-input-smoothness}
    Assume $\rho \neq 0$. Let $[\theta_\star]\in\mathcal M$ be an interpolation solution. Then,
    \begin{align}
    \frac{1}{n}\sum_{i=1}^n
    \|\nabla_x \bar f_{[\theta_\star]}(x_i)\|_2^2
    \le
    \frac{\Lambda_\star^2 C_{\mathrm{sec}}^2}{\rho^2}
    \,
    \mathrm{Flat}^Q([\theta_\star]).
    \end{align}
    The proof is provided in Appendix~\ref{app:proof-thm-general-quotient-flatness-implies-input-smoothness}.
\end{theorem}
$\Lambda_\star=\|W_{1,\star}\|_{\mathrm{op}}$ is the chain-rule factor that appears in passing from $W_1x$ to $x$. $\rho>0$ is needed to divide by $\|x_i\|_2$ uniformly; if some training input vanished, the bound would degenerate, which can be avoided w.l.o.g. by redefining the range of $x$.

Thm.~\ref{thm:general-quotient-flatness-implies-input-smoothness} is the symmetry-aware analogue of the flatness-to-smoothness step in \cite{ma2021linear}. It shows that\textit{ once function-preserving symmetries are removed}, flatness \textit{in the quotient geometry} controls the average input sensitivity of the learned function.

\subsection{Quotient Linearized SGD and Quotient Flatness Bound}

In this section, we derive an upper bound on the quotient flatness in terms of the batch size and learning rate.
In the next section, we derive a bound for the input smoothness.
We first explain how the symmetry-reduced local dynamics of SGD induces a \textit{linear recursion} on the tangent space of the quotient manifold.
Let $[\theta_\star] \in \mathcal M$ be an interpolation solution, and let
$\mathcal H := T_{[\theta_\star]}\mathcal M$
be the tangent space at $[\theta_\star]$, equipped with the inner product induced by $g^Q_{[\theta_\star]}$.
For each training point $x_i$, recall that
$a_i^Q := \nabla^Q \bar f_i([\theta_\star]) \in \mathcal H$.
Define the rank-one self-adjoint operator
$H_i^Q : \mathcal H \to \mathcal H$
by
$
H_i^Q \xi := g^Q_{[\theta_\star]}(a_i^Q,\xi)\,a_i^Q,
$
equivalently, in tensor notation,
$
H_i^Q = a_i^Q \otimes a_i^Q.
$

Consider one SGD step on the quotient manifold using a mini-batch
$I=(i_1,\dots,i_B)$
of size $B$ sampled with replacement.
Near $[\theta_\star]$, the quotient SGD map is locally given by
$
[q] \mapsto [q] - \frac{\eta}{B}\sum_{b=1}^B \nabla^Q \bar \ell_{i_b}([q]),
$
where $\bar \ell_i$ denotes the samplewise loss on the quotient manifold.
Therefore, the Jacobian of the quotient SGD map at $[\theta_\star]$ is
\begin{align}
M_I
:=
\mathrm{Id} - \frac{\eta}{B}\sum_{b=1}^B H_{i_b}^Q.
\end{align}

If $\xi_t \in \mathcal H$ denotes the infinitesimal quotient displacement of the iterate from $[\theta_\star]$ in local normal coordinates, then linearizing the quotient SGD map at $[\theta_\star]$ yields the linear recursion on the tangent space of the quotient manifold:
\begin{align}
    \xi_{t+1}
    =
    M_{I_t}\xi_t,
    \qquad
    \xi_t\in\mathcal H.
\end{align}
That is, the quotient linearized SGD is the tangent-space recursion obtained by retaining only the first-order term of the quotient SGD dynamics around $[\theta_\star]$.

\begin{definition}[One-step mean-square quotient linear stability] \label{def:general-ms-quotient-stability}
    We say that $[\theta_\star]$ is one-step mean-square quotient linearly stable if
    \begin{align}
    \mathbb E \|M_{I_t}\xi\|_{g^Q}^2 \le \|\xi\|_{g^Q}^2
    \qquad
    \text{for all } \xi \in \mathcal H,
    \end{align}
    where the expectation is over the random mini-batch $I_t$.
\end{definition}

Then, the linear recursion and stability leads to the following quotient flatness bound:

\begin{theorem}[Mean-square quotient stability implies quotient flatness bound] \label{thm:general-ms-stability-implies-flatness}
    Assume Definition~\ref{def:general-ms-quotient-stability}. Then
    $
    \mathrm{Flat}^Q([\theta_\star]) \le \frac{2B}{\eta}.
    $
    The proof is provided in Appendix~\ref{app:proof-thm-general-ms-stability-implies-flatness}.
\end{theorem}

The mean-square quotient linear stability assumption
$
\mathbb E\|M_I \xi\|^2 \le \|\xi\|^2
$
for all $\xi \in \mathcal H$, with
$
M_I = \mathrm{Id} - \eta A_I
$
and
$
A_I = \frac{1}{B}\sum_{b=1}^B H_{i_b}^Q,
$
is a local small-step stability condition for the quotient linearized stochastic gradient descent (SGD) dynamics. It is the quotient analogue of the classical gradient-descent stability requirement $ \eta \lambda_{\max} \le 2 $ for a quadratic model. In particular, a sufficient condition is
$
\lambda_{\max}(A_I) \le \frac{2}{\eta}
$
for every mini-batch $I$, since then $M_I^\ast M_I \preceq \mathrm{Id}$ batchwise. In the full-batch limit, this reduces to the familiar condition
$
\eta \lambda_{\max}(\bar H_Q) \le 2,
$
where $ \bar H_Q = \frac{1}{n}\sum_{i=1}^n H_i^Q $. Thus, the assumption is realistic as a local condition near a converged interpolation solution under sufficiently small learning rate, especially after quotienting out function-preserving symmetry directions, although it is not intended as a global condition along the entire training trajectory.


Thm.~\ref{thm:general-ms-stability-implies-flatness} is the symmetry-aware flatness-selection statement. It shows that if the quotient linearized SGD does not increase the expected squared tangent-space norm in one step, then the quotient flatness must be small, with the explicit scaling $B/\eta$. Its empirical validation is provided in Fig.~\ref{fig:flatness_bound}

\begin{figure}[tb]
    \begin{center}
    \includegraphics[width=0.6\textwidth]{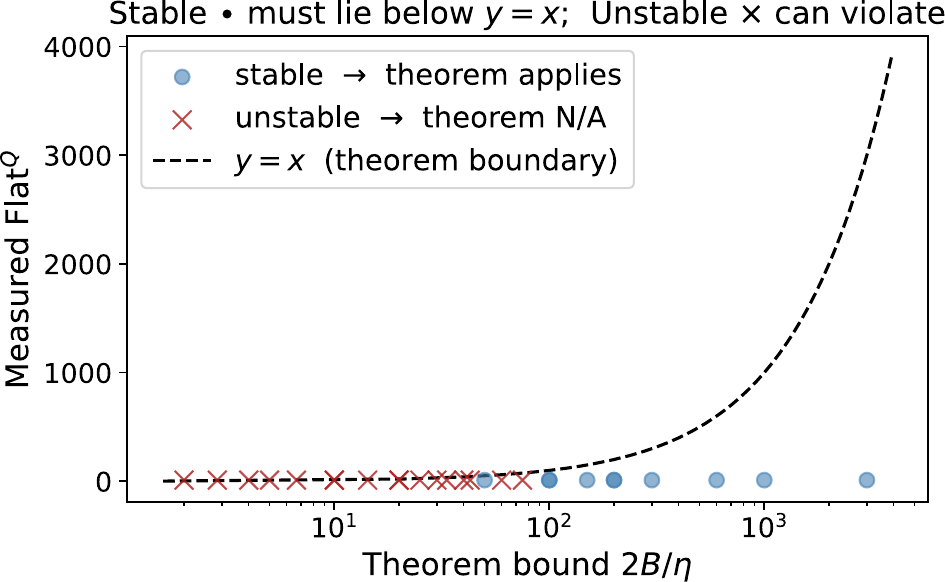}
    \end{center}
    \caption{\textbf{Flatness bound in Thm.~\ref{thm:general-ms-stability-implies-flatness}.}
          A two-layer ReLU student network
          $f(x)=\sum_{r=1}^{m}a_r\operatorname{ReLU}(w_r^\top x)$
          with $m=10$ neurons and input dimension $d=4$
          is trained on a teacher-generated regression dataset of $n=25$ samples
          (teacher width $m_{\rm t}=4$).
          After training to near-interpolation with SGD
          ($\ell_{\rm train}<10^{-7}$),
          the quotient flatness
          $\mathrm{Flat}^{Q}(\theta^*)$
          and the one-step mean-square quotient stability operator
          $\mathbb{E}_I[M_I^2]=I-2\eta\bar{A}+\eta^2\,\mathbb{E}[A_I^2]$
          are evaluated analytically for each combination of
          learning rate $\eta\in\{0.01,0.05,0.10,0.20,0.30,0.40,0.50,0.70,1.00\}$
          and batch size $B\in\{1,5,15\}$.
          Stability is assessed as $\lambda_{\max}(\mathbb{E}[M_I^2])\le 1$,
          computed on the active subspace of $A^Q$
          (rows of $A^Q$ projected onto its row space via SVD,
          threshold $10^{-6}\,\sigma_1$).
          \textbf{Blue circles}: stable $({\eta,B})$ pairs,
          for which the theorem guarantees
          $\mathrm{Flat}^Q\le 2B/\eta$ (below the diagonal $y=x$).
          \textbf{Red crosses}: unstable pairs, for which the bound may be violated.
          All $11$ stable pairs satisfy the bound, confirming the theorem.
    }
    \label{fig:flatness_bound}
\end{figure}

\subsection{Higher-order Quotient Moment Control and Bound for Input Smoothness}
Next, we derive a bound for the input smoothness.
Let $\mathrm{Sym}^k(\mathcal H)$ be the symmetric $k$-tensor space over $\mathcal H$. For the quotient linearized dynamics
$
\xi_{t+1}=M_{I_t}\xi_t,
$
define
$
M_t^{(k)} := \mathbb E[\xi_t^{\otimes k}].
$

\begin{proposition}[Closed quotient tensor recursion]
\label{prop:quotient-tensor-recursion}
For each $k\ge 1$, there exists a linear operator
$
\mathcal T_k := \mathbb E[M_{I_t}^{\otimes k}]
$
on $\mathrm{Sym}^k(\mathcal H)$ such that
\begin{align}
M_{t+1}^{(k)} = \mathcal T_k M_t^{(k)}.
\end{align}
The proof is provided in Appendix~\ref{app:proof-prop-quotient-tensor-recursion}.
\end{proposition}

The closed tensor recursion in Prop.~\ref{prop:quotient-tensor-recursion} separates two roles of the quotient linearized dynamics. First, as used in Thm.~\ref{thm:general-ms-stability-implies-flatness}, one-step mean-square non-expansiveness bounds the trace of the quotient Hessian and hence quotient flatness. Second, tensorized stability controls higher empirical moments of the samplewise quotient gradients ${a_i^Q}{i=1}^n$. These moments are unnecessary for the basic flatness bound, but are needed to control stronger input-side smoothness quantities, such as empirical $2k$-th moments of $\|\nabla_x \bar f_{[\theta_\star]}(x_i)\|_2$.

The next step extends this second-order principle to higher orders. As in Thm.~\ref{thm:general-ms-stability-implies-flatness}, if the quotient linearized SGD dynamics do not amplify tangent perturbations in the relevant moment, then the rank-one operators $a_i^Q \otimes a_i^Q$ must be controlled on average. Since the target quantity is now a degree-$2k$ polynomial in tangent perturbations rather than a quadratic form on $H$, the tensor recursion in Prop.~\ref{prop:quotient-tensor-recursion} provides the required closure, and the following stability assumption yields explicit empirical moment bounds.

\begin{definition}[One-step $2k$-mean quotient linear stability]
Let $k \ge 1$. We say that $[\theta_\star]$ is one-step $2k$-mean quotient
linearly stable if
\begin{align}
\mathbb E_I \|M_I \xi\|_{g^Q}^{2k}
\le
\|\xi\|_{g^Q}^{2k}
\qquad
\text{for every } \xi \in H ,
\end{align}
where $H=T_{[\theta_\star]}M$, $M_I=\mathrm{Id}-\eta A_I$, and
\begin{align}
A_I
=
\frac{1}{B}
\sum_{b=1}^B
a^Q_{i_b} \otimes a^Q_{i_b}.
\end{align}
\end{definition}

\begin{theorem}[Higher-order quotient moment bound] \label{thm:sharp-quotient-moment-bound}
    Assume that $[\theta_\star]$ is one-step $2k$-mean quotient linearly stable.
    Let $\{e_j\}_{j=1}^{d_Q}$ be any $g^Q$-orthonormal basis of
    $H=T_{[\theta_\star]}M$. Then, for every $j=1,\ldots,d_Q$,
    \begin{align}
    \frac{1}{n}
    \sum_{i=1}^n
    |\langle a_i^Q,e_j\rangle_{g^Q}|^{2k}
    \le
    \frac{2^k B^{k-1}}{\eta^k}.
    \end{align}
    Consequently,
    \begin{align}
    \frac{1}{n}
    \sum_{i=1}^n
    \|a_i^Q\|_{\ell_{2k}(e)}^{2k}
    \le
    d_Q
    \frac{2^k B^{k-1}}{\eta^k},
    \end{align}
    where
$
    \|v\|_{\ell_{2k}(e)}^{2k}
    :=
    \sum_{j=1}^{d_Q}
    |\langle v,e_j\rangle_{g^Q}|^{2k}.
$
    Moreover, for the Riemannian norm,
    \begin{align}
    \frac{1}{n}
    \sum_{i=1}^n
    \|a_i^Q\|_{g^Q}^{2k}
    \le
    d_Q^k
    \frac{2^k B^{k-1}}{\eta^k}.
    \end{align}
    The proof is provided in Appendix~\ref{app:proof-thm-sharp-quotient-moment-bound}.
\end{theorem}

Thm.~\ref{thm:sharp-quotient-moment-bound} is still an intrinsic statement on the quotient tangent space: it
bounds the empirical higher moments of the quotient gradients $a_i^Q$. By itself,
it does not yet give a statement about the regularity of the learned predictor as
a function of the input variable $x$. To obtain such a statement, we must combine
Thm.~\ref{thm:sharp-quotient-moment-bound} with the geometric comparison used in Thm.~\ref{thm:general-quotient-flatness-implies-input-smoothness}. That comparison
transfers quotient-gradient control to first-layer parameter-gradient control
through the local gauge section, and then transfers first-layer control to input
derivatives through the multiplicative representation
$f_{s([\theta])}(x)=\widetilde f(W_1x,\vartheta)$. The following corollary is
therefore a direct synthesis of the dynamical moment estimate and the
quotient-to-input comparison.

\begin{corollary}[Higher-order input smoothness]
    \label{cor:general-higher-order-input-smoothness}
    Assume the hypotheses of Theorem 4.1 and the one-step $2k$-mean quotient
    linear stability condition. Then
    \begin{align}
    \frac{1}{n}
    \sum_{i=1}^n
    \|\nabla_x \bar f_{[\theta_\star]}(x_i)\|_2^{2k}
    \le
    \left(
    \frac{\Lambda_\star C_{\mathrm{sec}}}{\rho}
    \right)^{2k}
    d_Q^k
    \frac{2^k B^{k-1}}{\eta^k}.
    \end{align}
    The proof is provided in Appendix~\ref{app:proof-cor-general-higher-order-input-smoothness}.
\end{corollary}

The results in this subsection complete the optimization-to-smoothness step.
Thm.~\ref{thm:general-ms-stability-implies-flatness} shows that quotient linear stability selects interpolating solutions with small quotient flatness.
Thm.~\ref{thm:sharp-quotient-moment-bound} and Cor.~\ref{cor:general-higher-order-input-smoothness} extend this implication: stronger moment stability of the quotient linearized dynamics gives stronger empirical control of input derivatives.
These estimates are local both in parameter space, since they use the linearized quotient SGD dynamics near an interpolation solution, and in input space, since they control derivatives only at the training inputs.

The next section converts these local estimates into a population generalization bound.
Under a balanced local covering condition, most test mass lies near the training samples.
If the target is Lipschitz and the learned quotient model is locally affine on these neighborhoods, the above input-smoothness bounds control the covered-region prediction error, while boundedness controls the uncovered mass.
Thus quotient flatness, and hence quotient stability, yields generalization.

\section{Quotient Flatness Implies Generalization under Local Covering}
We now derive the relationship that \textit{the quotient flatness implies generalization}.
Let $\varepsilon>0$, $\delta\in[0,1)$, and $\nu\ge 1$. We assume that there exist measurable sets $V_1,\dots,V_n$ such that
\begin{align}
&V_i \subset B(x_i,\varepsilon)
\ \
\text{for all } i,
\qquad
V_i \cap V_j = \varnothing
\ \
\text{for } i\neq j, \nonumber \\
&\mathbb P_X\Bigl(\bigcup_{i=1}^n V_i\Bigr)\ge 1-\delta,
\qquad
\mathbb P_X(V_i)\le \frac{\nu}{n}
\ \
\text{for all } i,
\end{align}
where for $x \in \mathcal X \subset \mathbb R^d$ and $\varepsilon>0$,
$
B(x,\varepsilon)
:=
\{y\in \mathcal X : \|y-x\|_2 < \varepsilon\}
$
denotes the Euclidean ball in the input space.
We refer to the covering hypothesis above as the \textit{$(\varepsilon,\delta,\nu)$-balanced local covering condition}.
Assume also that the target function $f^\star$ is $L_\star$-Lipschitz and that both $\bar f_{[\theta_\star]}$ and $f^\star$ are uniformly bounded by $M$ on $\mathcal X$.
Finally, assume that on each neighborhood $B(x_i,\varepsilon)$, the quotient model is affine with gradient $\nabla_x \bar f_{[\theta_\star]}(x_i)$. This holds, for example, in the two-layer ReLU case if the activation pattern is stable in those neighborhoods.

\begin{theorem}[Quotient flatness implies generalization]
    \label{thm:general-quotient-flatness-implies-generalization}
    Under the assumptions above, define $\mathcal E_{\mathrm{gen}}([\theta_\star])
        :=
        \mathbb E_{X\sim \mathbb P_X}
        \bigl[
        (\bar f_{[\theta_\star]}(X)-f^\star(X))^2
        \bigr]$, and we have
    \begin{align} \label{eq:main_theorem}
        \mathcal E_{\mathrm{gen}}([\theta_\star])
        \le
        2\nu \varepsilon^2
        \left(
        \frac{\Lambda_\star^2 C_{\mathrm{sec}}^2}{\rho^2}
        \mathrm{Flat}^Q([\theta_\star])
        +
        L_\star^2
        \right)
        +
        4M^2 \delta.
    \end{align}
    The proof is provided in Appendix~\ref{app:proof-thm-general-quotient-flatness-implies-generalization}.
\end{theorem}

The balanced local covering condition requires most test mass to lie near the training samples without concentrating around only a few of them. This is plausible for locally clustered data, but restrictive when high-dimensional data are diffuse, making small neighborhoods around training samples low-mass. The Lipschitz condition on $f^\star$ controls target variation within neighborhoods, the local affinity assumption controls predictor variation via its input gradient, and the bound $M$ controls the uncovered mass of size at most $\delta$. Together with the quotient-flatness-to-input-smoothness estimate, these assumptions convert symmetry-aware flatness into a population generalization bound.

Thm.~\ref{thm:general-quotient-flatness-implies-generalization} completes the symmetry-aware pipeline from quotient flatness to generalization. The result separates the role of geometry, through the quotient flatness term, from the role of data coverage, through the local covering term.
The following corollary is immediate.

\begin{corollary}[From quotient stability to generalization]
\label{cor:general-stability-to-generalization}
Under the assumptions of Thms.~\ref{thm:general-ms-stability-implies-flatness} and~\ref{thm:general-quotient-flatness-implies-generalization},
\begin{align}
\mathcal E_{\mathrm{gen}}([\theta_\star])
\le
2\nu \varepsilon^2
\left(
\frac{2B}{\eta}
\frac{\Lambda_\star^2 C_{\mathrm{sec}}^2}{\rho^2}
+
L_\star^2
\right)
+
4M^2\delta.
\end{align}
The proof is provided in Appendix~\ref{app:proof-cor-general-stability-to-generalization}.
\end{corollary}


Fig.~\ref{fig:quotient-flatness-generalization-experiment} provides an empirical illustration of the generalization bound in
Thm.~\ref{thm:general-quotient-flatness-implies-generalization}. In the teacher--student two-layer ReLU setting, the empirical
generalization error remains below the theoretical right-hand side of the generalization bound, and the intermediate relation between
quotient flatness and input smoothness is also observed.


\begin{figure}[htbp]
  \centering
  \includegraphics[width=1.0\textwidth]{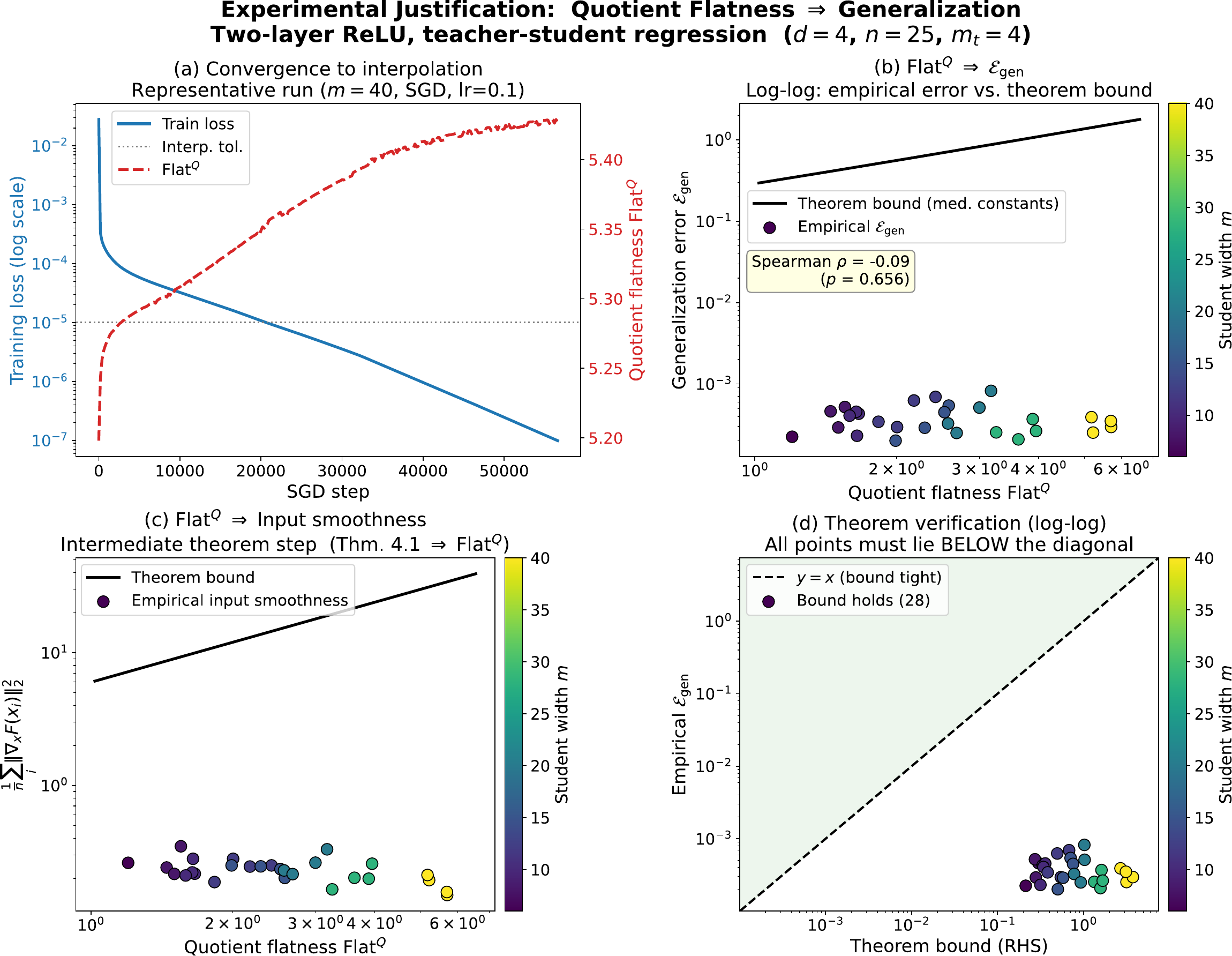}
  \caption{%
\textbf{Experimental verification of Thm.~\ref{thm:general-quotient-flatness-implies-generalization}.}
We test the quotient-flatness generalization bound in a teacher--student regression task with two-layer ReLU networks. In this two-layer specialization, the factor
$\Lambda_\star^2 C_{\mathrm{sec}}^2$ is evaluated by the explicit constant
$\kappa_\ast \max(1,A_\ast^2)$.
The teacher has input dimension $d=4$ and width $m_t=4$; the training set contains $n=25$ points sampled uniformly from $\mathcal{S}^{d-1}$.
Student networks of width $m \in \{6,8,10,12,15,20,28,40\}$ are trained by full-batch SGD with learning rate $0.1$ for up to $2\times10^5$ steps, until $\hat{\mathcal{L}} \le 5\times10^{-5}$.
The test distribution $P_X$ is uniform on $\bigcup_{i=1}^{n} B(x_i,\varepsilon)$ with $\varepsilon=0.15$, so the $(\varepsilon,\delta,\nu)$-balanced covering condition holds exactly with $\nu=1$ and $\delta=0$.
We use $4$ seeds per width; $28$ of $32$ runs reach interpolation.
\textbf{(a)}
For a representative run with $m=40$, the training loss decreases monotonically, while the quotient flatness $\mathrm{Flat}^Q$ increases, indicating increasing quotient-space curvature during fitting.
\textbf{(b)}
Generalization error $\mathcal{E}_\mathrm{gen}$ versus $\mathrm{Flat}^Q$ on a log--log scale.
Each point is one converged model, colored by $m$.
The solid curve is the bound~\eqref{eq:main_theorem} evaluated at median geometric constants.
Specifically, $\kappa_\ast=\|U_\ast\|_{\mathrm{op}}^2$, where the rows of $U_\ast$ are $u_r=w_r/\|w_r\|$; $A_\ast=\max_r |\alpha_r|$, with $\alpha_r=a_r\|w_r\|$; $\rho=\min_i\|x_i\|=1$; and $M$ uniformly bounds $|F_\theta(x)|$ and $|f_\ast(x)|$ on $\operatorname{supp}(P_X)$.
Empirically, $\kappa_\ast \in [3.6,14.5]$ with median $6.0$, $A_\ast \in [0.26,0.79]$ with median $0.56$, $\rho=1$, and $M\approx0.49$.
All points lie well below the bound.
\textbf{(c)}
Empirical input smoothness $\tfrac{1}{n}\sum_i \|\nabla_x F(x_i)\|^2$ versus $\mathrm{Flat}^Q$ on a log--log scale.
The bound $\kappa_\ast \max(1,A_\ast^2)\rho^{-2}\mathrm{Flat}^Q$ holds for all models, verifying the intermediate implication from quotient flatness to input smoothness.
\textbf{(d)}
Direct verification of the theorem bound: empirical $\mathcal{E}_\mathrm{gen}$ versus the theorem right-hand side on a log--log scale.
All $28$ converged models lie below the diagonal $y=x$, confirming the bound.
  }
  \label{fig:quotient-flatness-generalization-experiment}
\end{figure}

\section{Limitations} \label{sec:limitations}

Several extensions remain open.
First, cross-entropy introduces a curvature-residual term that breaks the exact interpolation identity used here.
Second, the global permutation quotient is naturally stratified, so extending the theory beyond the regular stratum may require orbifold or stratified geometry.
Finally, our empirical evidence remains limited: reproducing flatness--generalization phenomena in deep architectures is difficult and closely tied to the well-known \textit{flatness criticism}.
Although these issues are left unresolved, our theorems identify precisely where the new technical obstacles arise.

\begin{credits}
\subsubsection{\ackname} 
I thank Toshinori Araki, my manager at NEC Corporation, for his dedicated support.

\subsubsection{\discintname}
N/A
\end{credits}

%
%
\bibliographystyle{splncs04}
\bibliography{references}

\appendix
\section{Proofs}
\label{app:deferred-proofs}

\subsection{Proof of Prop.~\ref{prop:general-quotient-flatness-identity}}
\label{app:proof-prop-general-quotient-flatness-identity}



\begin{proof}
For each $i$, define
$
\bar f_i([\theta]) := \bar f_{[\theta]}(x_i),
\qquad
r_i([\theta]) := \bar f_i([\theta]) - y_i,
\qquad
\phi_i([\theta]) := \frac12 r_i([\theta])^2.
$
Then
$
\bar L = \frac1n \sum_{i=1}^n \phi_i.
$
Let $h:\mathbb R \to \mathbb R$ be given by $h(t)=\frac12 t^2$, so that $\phi_i=h\circ r_i$. By the Riemannian Hessian chain rule,
$
\mathrm{Hess}_{\mathcal M}\phi_i
=
h''(r_i)\,dr_i \otimes dr_i + h'(r_i)\,\mathrm{Hess}_{\mathcal M}r_i.
$
Since $h'(t)=t$ and $h''(t)=1$, this becomes
$
\mathrm{Hess}_{\mathcal M}\phi_i
=
dr_i \otimes dr_i + r_i\,\mathrm{Hess}_{\mathcal M}r_i.
$
Because $y_i$ is constant, we have $dr_i=d\bar f_i$ and $\mathrm{Hess}_{\mathcal M}r_i=\mathrm{Hess}_{\mathcal M}\bar f_i$. Hence
$
\mathrm{Hess}_{\mathcal M}\phi_i
=
d\bar f_i \otimes d\bar f_i + (\bar f_i-y_i)\,\mathrm{Hess}_{\mathcal M}\bar f_i.
$
Now let $[\theta_\star]$ be an interpolation solution. Then $\bar f_i([\theta_\star])=y_i$ for every $i$, so $r_i([\theta_\star])=0$. Therefore
$
\mathrm{Hess}_{\mathcal M}\phi_i([\theta_\star])
=
d\bar f_i|_{[\theta_\star]} \otimes d\bar f_i|_{[\theta_\star]}.
$
Let $\{e_\alpha\}_{\alpha=1}^{\dim \mathcal M}$ be a $g^Q$-orthonormal basis of $T_{[\theta_\star]}\mathcal M$. Taking the $g^Q$-trace gives
$
\mathrm{Tr}_{g^Q}\bigl(\mathrm{Hess}_{\mathcal M}\phi_i([\theta_\star])\bigr)
=
\sum_{\alpha=1}^{\dim \mathcal M}
\bigl(d\bar f_i|_{[\theta_\star]}(e_\alpha)\bigr)^2.
$
By the definition of the Riemannian gradient,
$
d\bar f_i|_{[\theta_\star]}(e_\alpha)
=
g^Q\bigl(\nabla^Q \bar f_i([\theta_\star]),e_\alpha\bigr).
$
Hence
$
\mathrm{Tr}_{g^Q}\bigl(\mathrm{Hess}_{\mathcal M}\phi_i([\theta_\star])\bigr)
=
\sum_{\alpha=1}^{\dim \mathcal M}
g^Q\bigl(\nabla^Q \bar f_i([\theta_\star]),e_\alpha\bigr)^2
=
\|\nabla^Q \bar f_i([\theta_\star])\|_{g^Q}^2
=
\|a_i^Q\|_{g^Q}^2.
$
Finally, by linearity of the Hessian and of the trace,
$
\mathrm{Flat}^Q([\theta_\star])
=
\mathrm{Tr}_{g^Q}\bigl(\mathrm{Hess}_{\mathcal M}\bar L([\theta_\star])\bigr)
=
\frac1n \sum_{i=1}^n
\mathrm{Tr}_{g^Q}\bigl(\mathrm{Hess}_{\mathcal M}\phi_i([\theta_\star])\bigr)
=
\frac1n \sum_{i=1}^n \|a_i^Q\|_{g^Q}^2.
$
This proves the claim.
\end{proof}

\subsection{Proof of Thm.~\ref{thm:general-quotient-flatness-implies-input-smoothness}} \label{app:proof-thm-general-quotient-flatness-implies-input-smoothness}

\begin{proof}

By the gauge-fixed representation, we may write
$s([\theta_\star]) = (W_{1,\star},\vartheta_\star)$ and
$f_{s([\theta_\star])}(x) = \tilde f(W_{1,\star}x,\vartheta_\star)$.
Let $z := W_{1,\star}x$ and define
$g_z(x) := \nabla_z \tilde f(z,\vartheta_\star)\vert_{z=W_{1,\star}x}$.
Then, the chain rule gives
$\nabla_x \bar f_{[\theta_\star]}(x) = W_{1,\star}^\top g_z(x)$,
hence
$\|\nabla_x \bar f_{[\theta_\star]}(x)\|_2
\le
\|W_{1,\star}\|_{\mathrm{op}} \|g_z(x)\|_2$.

On the other hand, differentiating $f_{(W_1,\vartheta_\star)}(x)=\tilde f(W_1x,\vartheta_\star)$ with respect to $W_1$ yields
$\nabla_{W_1} f_{(W_1,\vartheta_\star)}(x)=g_z(x)x^\top$.
Since this is a rank-one matrix, its Frobenius norm satisfies
$\|\nabla_{W_1} f_{(W_1,\vartheta_\star)}(x)\|_F
=
\|g_z(x)x^\top\|_F
=
\|g_z(x)\|_2 \|x\|_2$.
Therefore, for $x\neq 0$,
$\|g_z(x)\|_2
=
\|\nabla_{W_1} f_{(W_1,\vartheta_\star)}(x)\|_F / \|x\|_2$.
Combining the two estimates, we obtain
$\|\nabla_x \bar f_{[\theta_\star]}(x)\|_2
\le
(\|W_{1,\star}\|_{\mathrm{op}}/\|x\|_2)
\|\nabla_{W_1} f_{s([\theta_\star])}(x)\|_F$.
Applying this at $x=x_i$ proves the claimed inequality.
\end{proof}

\subsection{Proof of Thm.~\ref{thm:general-ms-stability-implies-flatness}}
\label{app:proof-thm-general-ms-stability-implies-flatness}


\begin{proof}
Expanding $\mathbb E\|M_{I_t}\xi\|^2 \le \|\xi\|^2$ yields the operator inequality
$
\mathbb E(M_{I_t}^\ast M_{I_t}) \preceq I.
$
Writing
$
A_t := \frac1B\sum_{b=1}^B H_{i_b}^Q,
$
we have $M_{I_t}=I-\eta A_t$, hence
$
0
\ge
-2\eta\,\mathbb E\,\mathrm{Tr}(A_t)
+
\eta^2 \mathbb E\,\mathrm{Tr}(A_t^2).
$
Now,
$
\mathbb E\,\mathrm{Tr}(A_t)
=
\mathrm{Tr}\left(\frac1n\sum_{i=1}^n H_i^Q\right)
=
\frac1n\sum_{i=1}^n \|a_i^Q\|^2
=
\mathrm{Flat}^Q([\theta_\star]).
$
Also,
$
\mathrm{Tr}(A_t^2)
=
\frac1{B^2}\sum_{b,c=1}^B \langle a_{i_b}^Q,a_{i_c}^Q\rangle^2
\ge
\frac1{B^2}\sum_{b=1}^B \|a_{i_b}^Q\|^4.
$
Taking expectation and using Cauchy--Schwarz gives
$
\mathbb E\,\mathrm{Tr}(A_t^2)
\ge
\frac1B
\left(
\frac1n\sum_{i=1}^n \|a_i^Q\|^2
\right)^2
=
\frac1B
\bigl(\mathrm{Flat}^Q([\theta_\star])\bigr)^2.
$
Substituting into the previous inequality yields
$
2\,\mathrm{Flat}^Q([\theta_\star])
\ge
\eta \frac1B \bigl(\mathrm{Flat}^Q([\theta_\star])\bigr)^2.
$
If $\mathrm{Flat}^Q([\theta_\star])=0$, the claim is trivial. Otherwise divide by $\mathrm{Flat}^Q([\theta_\star])$ to conclude
$\mathrm{Flat}^Q([\theta_\star]) \le \frac{2B}{\eta}$.

\end{proof}



\subsection{Proof of Prop.~\ref{prop:quotient-tensor-recursion}}
\label{app:proof-prop-quotient-tensor-recursion}


\begin{proof}
Since $M_{I_t}$ is independent of $\xi_t$,
$
M_{t+1}^{(k)}
=
\mathbb E[(M_{I_t}\xi_t)^{\otimes k}]
=
\mathbb E[M_{I_t}^{\otimes k}]\,\mathbb E[\xi_t^{\otimes k}]
=
\mathcal T_k M_t^{(k)}.
$
\end{proof}

\subsection{Proof of Thm.~\ref{thm:sharp-quotient-moment-bound}}
\label{app:proof-thm-sharp-quotient-moment-bound}


\begin{proof}
    Fix an arbitrary basis vector $e_j$ and write
 $
    z_i
    :=
    \langle a_i^Q,e_j\rangle_{g^Q},
    \qquad
    X_i
    :=
    z_i^2
    \ge 0 .
 $ 
 For a mini-batch $I=(i_1,\ldots,i_B)$ sampled with replacement, define
    \begin{align}
    Z_I
    :=
    \frac{1}{B}
    \sum_{b=1}^B
    X_{i_b}
    =
    \frac{1}{B}
    \sum_{b=1}^B
    \langle a_{i_b}^Q,e_j\rangle_{g^Q}^2 .
    \end{align}
    Since
$
    A_I
    =
    \frac{1}{B}
    \sum_{b=1}^B
    a_{i_b}^Q \otimes a_{i_b}^Q ,
$
    we have
$
    \langle A_I e_j,e_j\rangle_{g^Q}
    =
    Z_I .
$
    Therefore,
$
    \langle M_I e_j,e_j\rangle_{g^Q}
    =
    \langle (\mathrm{Id}-\eta A_I)e_j,e_j\rangle_{g^Q}
    =
    1-\eta Z_I .
$
    By Cauchy--Schwarz,
    \begin{align}
    |1-\eta Z_I|
    =
    |\langle M_I e_j,e_j\rangle_{g^Q}|
    \le
    \|M_I e_j\|_{g^Q}\|e_j\|_{g^Q}
    =
    \|M_I e_j\|_{g^Q}.
    \end{align}
    Since $\|e_j\|_{g^Q}=1$, the one-step $2k$-mean stability assumption gives
 $
    \mathbb E_I |1-\eta Z_I|^{2k}
    \le
    \mathbb E_I \|M_I e_j\|_{g^Q}^{2k}
    \le
    1 .
$
    We now bound the $k$-th moment of $Z_I$. By the triangle inequality in
    $L^{2k}$,
    \begin{align}
    \|\eta Z_I\|_{L^{2k}}
    &=
    \|1-(1-\eta Z_I)\|_{L^{2k}}  \le
    \|1\|_{L^{2k}}
    +
    \|1-\eta Z_I\|_{L^{2k}}  \le
    1+1
    =
    2 .
    \end{align}
    Hence
$
    \mathbb E_I Z_I^{2k}
    \le
    \left(\frac{2}{\eta}\right)^{2k}.
$
    By Lyapunov's inequality,
$
    \mathbb E_I Z_I^k
    \le
    \left(
    \mathbb E_I Z_I^{2k}
    \right)^{1/2}
    \le
    \left(\frac{2}{\eta}\right)^k .
$
    
    On the other hand, since all $X_i$ are nonnegative,
$
    Z_I^k
    =
    \left(
    \frac{1}{B}
    \sum_{b=1}^B
    X_{i_b}
    \right)^k
    \ge
    \frac{1}{B^k}
    \sum_{b=1}^B
    X_{i_b}^k .
$
    Taking expectation over the mini-batch sampled with replacement yields
    \begin{align}
    \mathbb E_I Z_I^k
    &\ge
    \frac{1}{B^k}
    \sum_{b=1}^B
    \mathbb E_{i_b} X_{i_b}^k  =
    \frac{1}{B^k}
    B
    \frac{1}{n}
    \sum_{i=1}^n
    X_i^k  
    =
    B^{1-k}
    \frac{1}{n}
    \sum_{i=1}^n
    |\langle a_i^Q,e_j\rangle_{g^Q}|^{2k}.
    \end{align}
    Combining the upper and lower bounds on $\mathbb E_I Z_I^k$, we obtain
 $ 
    B^{1-k}
    \frac{1}{n}
    \sum_{i=1}^n
    |\langle a_i^Q,e_j\rangle_{g^Q}|^{2k}
    \le
    \left(\frac{2}{\eta}\right)^k .
$
    Therefore,
$
    \frac{1}{n}
    \sum_{i=1}^n
    |\langle a_i^Q,e_j\rangle_{g^Q}|^{2k}
    \le
    \frac{2^k B^{k-1}}{\eta^k}.
$
    This proves the coordinate-wise bound.
    
    Summing the coordinate-wise bound over $j=1,\ldots,d_Q$ gives
    \begin{align}
    \frac{1}{n}
    \sum_{i=1}^n
    \|a_i^Q\|_{\ell_{2k}(e)}^{2k}
    &=
    \frac{1}{n}
    \sum_{i=1}^n
    \sum_{j=1}^{d_Q}
    |\langle a_i^Q,e_j\rangle_{g^Q}|^{2k}  \le
    d_Q
    \frac{2^k B^{k-1}}{\eta^k}.
    \end{align}
    Finally, for the Riemannian norm,
    \begin{align}
    \|v\|_{g^Q}^{2k}
    =
    \left(
    \sum_{j=1}^{d_Q}
    |\langle v,e_j\rangle_{g^Q}|^2
    \right)^k
    \le
    d_Q^{k-1}
    \sum_{j=1}^{d_Q}
    |\langle v,e_j\rangle_{g^Q}|^{2k}
    =
    d_Q^{k-1}
    \|v\|_{\ell_{2k}(e)}^{2k}.
    \end{align}
    Applying this to $v=a_i^Q$ and averaging over $i$ gives
 $
    \frac{1}{n}
    \sum_{i=1}^n
    \|a_i^Q\|_{g^Q}^{2k}
    \le
    d_Q^k
    \frac{2^k B^{k-1}}{\eta^k}.
$
    The proof is complete.
\end{proof}

\subsection{Proof of Cor.~\ref{cor:general-higher-order-input-smoothness}}
\label{app:proof-cor-general-higher-order-input-smoothness}


\begin{proof}
    By the pointwise estimate used in Theorem 4.1, for every training point $x_i$,
$
    \|\nabla_x \bar f_{[\theta_\star]}(x_i)\|_2
    \le
    \frac{\Lambda_\star C_{\mathrm{sec}}}{\rho}
    \|a_i^Q\|_{g^Q}.
$
    Raising both sides to the power $2k$ and averaging over $i$ gives
$
    \frac{1}{n}
    \sum_{i=1}^n
    \|\nabla_x \bar f_{[\theta_\star]}(x_i)\|_2^{2k}
    \le
    \left(
    \frac{\Lambda_\star C_{\mathrm{sec}}}{\rho}
    \right)^{2k}
    \frac{1}{n}
    \sum_{i=1}^n
    \|a_i^Q\|_{g^Q}^{2k}.
$
    Applying the Riemannian-norm version of the higher-order quotient moment
    bound yields the claim.
\end{proof}

\subsection{Proof of Thm.~\ref{thm:general-quotient-flatness-implies-generalization}}
\label{app:proof-thm-general-quotient-flatness-implies-generalization}


\begin{proof}
    For $x\in V_i$, by local affinity,
    $
    \bar f_{[\theta_\star]}(x)
    =
    \bar f_{[\theta_\star]}(x_i)
    +
    \nabla_x \bar f_{[\theta_\star]}(x_i)^\top (x-x_i).
    $
    Since $[\theta_\star]$ is an interpolation solution,
    $
    \bar f_{[\theta_\star]}(x_i)=y_i=f^\star(x_i).
    $
    Thus
    \begin{align}
    \bar f_{[\theta_\star]}(x)-f^\star(x)
    =
    \nabla_x \bar f_{[\theta_\star]}(x_i)^\top (x-x_i)
    +
    f^\star(x_i)-f^\star(x).
    \end{align}
    Using the Lipschitz continuity of $f^\star$ and $(a+b)^2 \le 2a^2+2b^2$, we obtain
$
    (\bar f_{[\theta_\star]}(x)-f^\star(x))^2
    \le
    2\varepsilon^2
    \left(
    \|\nabla_x \bar f_{[\theta_\star]}(x_i)\|_2^2
    +
    L_\star^2
    \right).
$
    Integrating over $V_i$, summing over $i$, and using $\mathbb P_X(V_i)\le \nu/n$ gives
    \begin{align}
    \int_{\cup_i V_i}
    (\bar f_{[\theta_\star]}(x)-f^\star(x))^2 d\mathbb P_X(x)
    \le
    2\nu\varepsilon^2
    \left(
    \frac1n\sum_{i=1}^n \|\nabla_x \bar f_{[\theta_\star]}(x_i)\|_2^2
    +
    L_\star^2
    \right).
    \end{align}
    On the complement, the uniform bound yields
    $
    \int_{\mathcal X\setminus \cup_i V_i}
    (\bar f_{[\theta_\star]}(x)-f^\star(x))^2 d\mathbb P_X(x)
    \le
    4M^2\delta.
    $
    Combining the two inequalities and applying Thm.~\ref{thm:general-quotient-flatness-implies-input-smoothness} proves the claim.
\end{proof}

\subsection{Proof of Cor.~\ref{cor:general-stability-to-generalization}}
\label{app:proof-cor-general-stability-to-generalization}


\begin{proof}
Substitute the quotient flatness bound from Thm.~\ref{thm:general-ms-stability-implies-flatness} into Thm.~\ref{thm:general-quotient-flatness-implies-generalization}.
\end{proof}

\end{document}